\documentclass{article}

\PassOptionsToPackage{numbers, compress}{natbib}

\usepackage[preprint]{neurips_2026}

\usepackage[utf8]{inputenc} 
\usepackage[T1]{fontenc}    
\usepackage{hyperref}       
\usepackage{url}            
\usepackage{booktabs}       
\usepackage{amsfonts}       
\usepackage{nicefrac}       
\usepackage{microtype}      

\usepackage{graphicx}
\usepackage{subcaption}

\usepackage{amsmath}
\usepackage{amssymb}
\usepackage{mathtools}
\usepackage{amsthm}
\usepackage{wrapfig}
\usepackage{algorithm}
\usepackage{algorithmic}

\usepackage[capitalize,noabbrev]{cleveref}

\theoremstyle{plain}
\newtheorem{theorem}{Theorem}[section]
\newtheorem{proposition}[theorem]{Proposition}

\theoremstyle{definition}

\theoremstyle{remark}

\usepackage[disable,textsize=tiny]{todonotes}

\author{%
  Shiyan Liu$^{1}$\thanks{Equal contribution.} \quad Bohan Tan$^{1}$\footnotemark[1] \quad Yaoxin Wu$^{2}$ \quad Yan Jin$^{1}$ \\[0.5ex]
  $^1$ Huazhong University of Science and Technology \\
  $^2$ Eindhoven University of Technology \\[0.5ex] 
  \texttt{\{shyl, bohant, jinyan\}@hust.edu.cn} \qquad \texttt{y.wu2@tue.nl}
}

\title{MViewRouter: Internalizing Geometric Equivariance via Multi-view Alternating Attention for Combinatorial Routing}

\usepackage{multirow}
\usepackage[table]{xcolor}
\usepackage{enumitem}

\begin{document}
\maketitle

\begin{abstract}
Combinatorial routing problems such as the Traveling Salesman Problem (TSP) and the Capacitated Vehicle Routing Problem (CVRP) are fundamental NP-hard problems with broad real-world applications. While recent deep reinforcement learning methods have shown promising performance, they typically handle geometric symmetries only through data augmentation, resulting in inconsistent decisions and limited generalization. To address this issue, we propose MViewRouter, a multi-view framework that internalizes geometric equivariance as a structural inductive bias to achieve invariant decision-making across routing problem variants. Our approach introduces a Multi-view Alternating Attention (MAA) mechanism that enables parallel processing over the $D_4$ symmetry group, alternating between intra-view relational modeling and inter-view feature alignment. Furthermore, we optimize the policy via Collective Policy Gradient Aggregation (CPGA), leveraging consensus gradients from multiple symmetric views to stabilize training and accelerate convergence. Experiments on TSP and CVRP benchmarks, as well as real-world TSPLIB instances, demonstrate that MViewRouter achieves competitive solution quality and strong zero-shot generalization.
\end{abstract}

\section{Introduction}
Combinatorial routing problems are cornerstones of combinatorial optimization, representing a class of NP-hard problems with profound implications across diverse industrial domains. The Traveling Salesman Problem (TSP) and the Capacitated Vehicle Routing Problem (CVRP) are two canonical representatives: TSP models the problem of finding the shortest Hamiltonian tour over a set of nodes, while CVRP additionally requires that deliveries respect vehicle capacity constraints. From optimizing last-mile delivery routes to scheduling pick-and-place operations in robotic manufacturing, the ability to rapidly, accurately solve such problems directly impacts operational costs and resource utilization \cite{baldacci2010some}. The exponential growth of the solution space renders classical approaches, including the exact solver Concorde \cite{applegate2006concorde} and the Lin-Kernighan heuristic (LKH) \cite{helsgaun2017extension}, computationally infeasible for large-scale or real-time applications, driving the development of learned heuristics \cite{bengio2021machine}.

In recent years, deep reinforcement learning has emerged as a promising paradigm for solving routing problems by learning heuristics directly from data. Constructive methods, such as the Attention Model (AM) \cite{kool2018attention} and Pointerformer \cite{jin2023pointerformer}, are particularly attractive for real-time deployment as they incrementally build solutions with high inference efficiency. Recent advancements like Policy Optimization with Multiple Optima (POMO) \cite{kwon2020pomo} and its Preference-Optimization (PO)-based extensions \cite{pan2025preference} have further narrowed the optimality gap by leveraging multi-start decoding strategies.

Despite advancements, most neural routing solvers treat instances as sequential coordinate sets, failing to internalize the geometric symmetries of the Euclidean plane. Standard Transformer-based \cite{vaswani2017attention} architectures often yield inconsistent embeddings for the same instance under different transformations of the $D_4$ dihedral group, the 8-element group of rotations and reflections that captures the full planar symmetry of Euclidean routing instances. This lack of geometric grounding forces a reliance on test-time augmentation (TTA) \cite{kwon2020pomo}, an unprincipled post-hoc patch that attempts to recover symmetry during inference rather than embedding it as a core inductive bias. Furthermore, single-view training overlooks rich supervisory signals from symmetric perspectives, leading to brittle policies sensitive to coordinate perturbations.

To address these challenges, we propose MViewRouter, a problem-agnostic framework establishing unified training and inference symmetry through two innovations. First, the Multi-view Alternating Attention (MAA) mechanism enables explicit feature-level communication between symmetric views \cite{wang2025vggt}. By alternating between intra-view modeling and inter-view alignment, the encoder learns a consensus representation naturally invariant to transformations. Second, we introduce Collective Policy Gradient Aggregation (CPGA), a symmetry-constrained policy gradient that enforces orbit-consistent (i.e., consistent across all 8 symmetric copies of the same instance) parameter updates by aggregating consensus gradients across all $D_4$ views in a single step \cite{yu2020gradient}. An orbit-conditioned view-specific baseline filters non-equivariant noise and prevents gradient domination by geometrically easier views. Consequently, MViewRouter ensures geometric invariance is an intrinsic policy property rather than a fragmented heuristic \cite{bronstein2021geometric}.

The main contributions of this work are summarized as follows:

\begin{itemize}
\vspace{-0.3em}
    \item We develop a Multi-view Alternating Attention mechanism that enables parallel processing of the $D_4$ symmetry group. By alternating between intra-view topology and inter-view alignment, the encoder learns equivariant representations grounded in geometric symmetry.
    \item We introduce Collective Policy Gradient Aggregation, a training paradigm treating symmetric views as concurrent objectives. By aggregating consensus gradients, CPGA accelerates convergence through denser supervision and enforces consistent, optimal decision-making orbits in the entire group.
    \item We validate MViewRouter on both TSP and CVRP benchmarks, as well as real-world TSPLIB instances, demonstrating its competitive performance, strong zero-shot generalization, and broad applicability beyond a single problem type.
\end{itemize}

\section{Related Work}

\subsection{Neural Constructive Solvers}
Neural constructive solvers have gained significant traction for solving the TSP. Early pioneers like Pointer Network \cite{vinyals2015pointer}, Neural CO \cite{bello2016neural}, and AM \cite{kool2018attention} utilized sequence-to-sequence and Transformer-based \cite{vaswani2017attention} architectures to construct solutions autoregressively, with subsequent work extending RL-based construction to heterogeneous vehicle routing problem variants \cite{nazari2018reinforcement}. To further enhance performance and reduce gradient variance, POMO \cite{kwon2020pomo} introduced a multiple-greedy rollout strategy, which has since become a cornerstone baseline in the field. 

Building on this, several works have focused on improving representation or search strategies. For instance, Pointerformer \cite{jin2023pointerformer} refines the attention mechanism for more expressive node features, Multi-Decoder AM \cite{xin2021multi} and Poppy \cite{grinsztajn2024winner} foster solution diversity through multi-policy ensembles, and search-based refinement \cite{choo2022simulation,hottung2021efficient} improves quality through guided exploration. Improvement-based approaches \cite{wu2021learning,lu2019learning,hottung2020neural,ma2024learning} complement these by iteratively refining existing solutions. To address the challenges of large-scale instances, DIFUSCO \cite{sun2023difusco} and H-TSP \cite{pan2023h} advance this frontier by leveraging diffusion models and hierarchical reinforcement learning, respectively, while NeuroLKH \cite{xin2021neurolkh,zheng2021combining} integrates neural guidance into classical LKH search. Transitioning from architecture to learning paradigms, POMO has recently been treated as a versatile framework. Prominent examples include preference optimization \cite{pan2025preference}, which transforms quantitative rewards into qualitative preference signals through statistical comparison and integrates local search during fine-tuning to help the policy escape local optima, and BOPO \cite{liao2025bopo}, a complementary preference-based variant.

While constructive methods focus on policy refinement or single-view architectures, our work explores a holistic multi-view framework. Unlike conventional approaches, our method enforces inter-view interactions during the encoding stage to achieve unified training and inference. By grounding diverse spatial perspectives into a consistent representation, MViewRouter ensures that geometric invariance is maintained as a fundamental structural property throughout the entire model lifecycle.

\subsection{Geometric Invariance}
Euclidean routing problems possess inherent geometric symmetries, which are vital for robust neural combinatorial optimization. Existing approaches address these symmetries at different levels of the model. Sym-NCO \cite{kim2022sym} enforces symmetry as an \emph{output-level} constraint via symmetry-aware regularization that penalizes solution discrepancies across transformed views. While effective, this leaves the feature space unconstrained: the encoder can still produce divergent embeddings for geometrically equivalent instances. At the architectural level, INVIT \cite{fang2024invit} mitigates embedding aliasing by designing an invariant encoder, but its invariance operates within a single view and does not exploit complementary information across multiple symmetric views. BQ-NCO \cite{drakulic2023bq} takes a complementary MDP-level approach, reducing the state space via bisimulation quotienting to leverage problem symmetries, though without explicit feature-level alignment across views.

In contrast, MViewRouter internalizes geometric symmetry as a structural inductive bias \cite{bronstein2021geometric,satorras2021n} rather than a post-hoc constraint. Unlike Sym-NCO's output-level regularization and INVIT's single-view invariance, MAA enforces \emph{feature-level} consistency via inter-view fiber attention before decoding. Unlike BQ-NCO's MDP-level state reduction, MViewRouter explicitly aligns latent representations across all $K$ symmetric views, achieving $D_4$-invariance as an intrinsic property of the encoder.

\section{Problem Formulation}
\label{sec:problem}

We consider two canonical combinatorial routing problems that serve as testbeds for our framework.

\paragraph{Traveling Salesman Problem (TSP).}
The TSP is defined on a complete undirected graph $G = (V, E)$, where $V = \{v_1, \dots, v_n\}$ denotes the set of $n$ nodes. Each node $v_i$ is represented by its 2D Euclidean coordinates, normalized to $[0, 1]^2$. The edge cost $c(v_i, v_j) = \lVert v_i - v_j \rVert_2$. A feasible solution is a Hamiltonian cycle $\tau$ that visits each node exactly once.
Given a tour $(v_{\sigma_1} v_{\sigma_2} ... v_{\sigma_n}), v_{\sigma_i} \in V$, the objective is to minimize the total tour length:
\begin{equation}
\label{Eq1}
    L_{\text{TSP}}(\tau) = \sum_{i=1}^{n-1} c(v_{\sigma_i}, v_{\sigma_{i+1}}) + c(v_{\sigma_n}, v_{\sigma_1}).
\end{equation}
\vspace{-1.5em}

\paragraph{Capacitated Vehicle Routing Problem (CVRP).}
CVRP extends TSP by introducing a depot $v_0$ and capacity constraints. The node set is $V = \{v_0, v_1, \dots, v_n\}$, where $v_0$ is the depot and $v_1, \dots, v_n$ are customers, each with a demand $d_i \in \mathbb{Z}_{>0}$. A fleet of homogeneous vehicles, each with capacity $Q$, must serve all customers. A feasible solution is a set of routes $\mathcal{R} = \{r_1, \dots, r_m\}$, where each route starts and ends at $v_0$, covers a disjoint subset of customers, and satisfies $\sum_{v_i \in r_j} d_i \leq Q$. Let $E(r)$ denote the ordered edge set of route $r$ (including the depot-to-first and last-to-depot legs). The objective is to minimize the total routing cost:
\begin{equation}
\label{Eq2}
    L_{\text{CVRP}}(\mathcal{R}) = \sum_{r \in \mathcal{R}} \sum_{(u,\,v) \in E(r)} c(u, v).
\end{equation}
\vspace{-1.5em}

Both problems share the same $D_4$ geometric symmetry structure: any rotation or reflection applied to all node coordinates yields an equivalent instance with identical optimal cost. This shared symmetry makes them natural testbeds for our geometric equivariance framework.

\begin{figure*}[!t]
    \centering
    \includegraphics[width=0.95\linewidth]{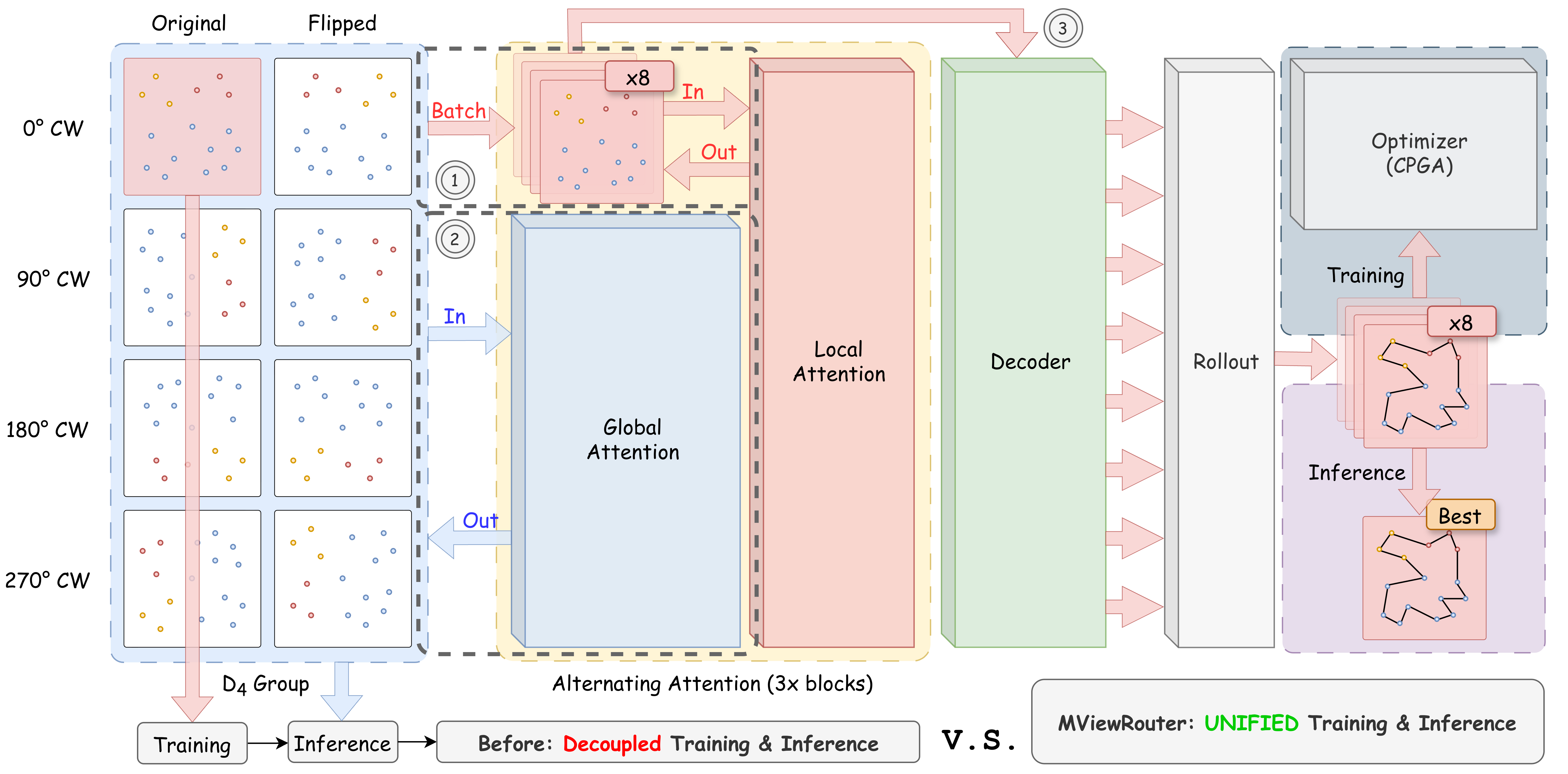}
    \caption{Overview of the MViewRouter architecture.}
    \label{fig:overall_architecture}
    \vspace{-1.5em}
\end{figure*}

\section{The Proposed MViewRouter Method}
\label{sec:method}
We propose a multi-view geometric equivariance framework, MViewRouter, that internalizes spatial symmetry as a structural inductive bias. MViewRouter enables cross-view information flow via Multi-view Alternating Attention (MAA) and optimizes the policy through Collective Policy Gradient Aggregation (CPGA), as illustrated in Figure~\ref{fig:overall_architecture}. Operating at the encoder and training-objective levels, MViewRouter is problem-agnostic: the same MAA and CPGA components apply to any routing problem with Euclidean node coordinates, including both TSP and CVRP. This unified multi-view, problem-agnostic design ensures that the geometric invariance used for TTA is explicitly grounded during both encoding and training.

\subsection{Multi-View Alternating Attention}
\label{sec:maa}
Given a routing instance $\mathbf{X} \in \mathbb{R}^{n \times 2}$, we generate $K=8$ augmented views $\{\mathbf{X}^{(k)} = g_k(\mathbf{X})\}_{k=1}^{K}$ by applying each element of the $D_4$ dihedral group pointwise (see Appendix~\ref{app:d4} for the explicit list of transformations), then embed them via a shared linear projection $f_\phi$ into a unified tensor $\mathbf{H}_0 \in \mathbb{R}^{K \times n \times d}$. To enforce latent consistency across these views, we introduce a MAA encoder \cite{wang2025vggt} consisting of $L$ blocks, each alternating between \textit{local} (intra-view) and \textit{global} (inter-view) reasoning.

\paragraph{Local Attention (Intra-view Topology).} 
The local attention layer captures node-to-node spatial dependencies within each view independently. For a given block $l$, the input is the multi-view embedding tensor $\mathbf{H}_{l-1} \in \mathbb{R}^{K \times n \times d}$. The intra-view relational structure for each view $k \in \{1,\dots,K\}$ is modeled as:
\begin{equation}
    \hat{\mathbf{H}}^{(k)}_l = \text{MHA}_{\text{local}}(\mathbf{H}^{(k)}_{l-1}),
\end{equation}
where $\text{MHA}(\cdot)$ denotes the Multi-Head Attention \cite{vaswani2017attention} mechanism with $M$ parallel heads. Given the attention weight matrices $\mathbf{W}^Q$, $\mathbf{W}^K$, and $\mathbf{W}^V$, each head $m$ takes $\mathbf{H}$ as input to compute:
\begin{equation}
    \text{Head}_m(\mathbf{H}) = \mathrm{softmax} \left( \frac{(\mathbf{H}\mathbf{W}^Q_m) (\mathbf{H}\mathbf{W}^K_m)^\top}{\sqrt{d_h}} \right) (\mathbf{H}\mathbf{W}^V_m),
\end{equation}
where $d_h = d/M$ is the dimension of each attention head. The local attention operation preserves the relative spatial topology of nodes under the specific geometric transformation $g_k$, ensuring that intra-view features remain locally consistent.

\paragraph{Global Attention (Inter-view Alignment).} 
To enforce equivariance across the symmetry group $\mathcal{G}$, we introduce interaction across the view dimension. Since the geometric transformations $g_k \in D_4$ are applied pointwise to coordinates, the node indexing remains invariant across views. Leveraging this correspondence, for each node $i \in \{1,\dots,n\}$, we define a cross-view fiber $\mathbf{F}_{i,l} \in \mathbb{R}^{K \times d}$ by collecting the embeddings of the same physical node $i$ from all $K$ views:
\begin{equation}
    \mathbf{F}_{i,l} = \left[ \hat{\mathbf{H}}^{(1)}_{l,i}, \hat{\mathbf{H}}^{(2)}_{l,i}, \dots, \hat{\mathbf{H}}^{(K)}_{l,i} \right]^\top.
\end{equation}
The inter-view alignment is then achieved by applying a shared $\mathrm{MHA}$ across the fiber dimension to update the node embeddings:
\begin{equation}
    \mathbf{F}'_{i,l} = \text{MHA}_{\text{global}}(\mathbf{F}_{i,l}), \quad \text{where } \mathbf{F}'_{i,l} = \left[ \mathbf{H}^{(1)}_{l,i}, \dots, \mathbf{H}^{(K)}_{l,i} \right]^\top.
\end{equation}
By applying attention across the view dimension, the encoder facilitates communication between different orientations of the same node. This regularizes the latent representations to be consistent under spatial transformations, bridging the gap between training and inference by ensuring that geometric equivariant features are explicitly aligned \cite{arnab2021vivit}.

\begin{proposition}[MAA Invariance at the Uniform Fixed Point]
\label{prop:maa_invariance}
Let $f_\phi$ be the shared linear projection and let $\mathrm{MHA}_{\mathrm{local}}$ have fixed weights. If the inter-view global attention weights are uniform, i.e., $\alpha_{j \to i}^{(k)} = 1/K$ for all $i, j, k$, then for any group element $g \in D_4$ inducing a permutation $\sigma_g$ on the $K$ views, the output fiber embedding satisfies
\begin{equation}
    \mathbf{H}^{(k)}_{l,i} = \mathbf{H}^{(\sigma_g(k))}_{l,i} \quad \forall\, i,\, k,
\end{equation}
yielding a $D_4$-invariant node representation.
\end{proposition}

\begin{proof}
Applying $g$ to the input instance permutes the views as $\mathbf{X}^{(k)} \mapsto \mathbf{X}^{(\sigma_g(k))}$. Since $f_\phi$ is a shared linear projection, the initial embeddings are correspondingly permuted: $\mathbf{H}^{(k)}_0 \mapsto \mathbf{H}^{(\sigma_g(k))}_0$. Because $\mathrm{MHA}_{\mathrm{local}}$ operates independently per view with fixed weights, this permutation propagates through the local attention, yielding $\hat{\mathbf{H}}^{(k)}_{l,i} \mapsto \hat{\mathbf{H}}^{(\sigma_g(k))}_{l,i}$ for all layers $l$. Consequently, applying $g$ permutes the fiber $\mathbf{F}_{i,l}$ by $\sigma_g$. Under uniform attention weights $\alpha_{j\to i}^{(k)} = 1/K$, the global MHA reduces to
\begin{equation}
    \mathbf{F}'_{i,l} = \frac{1}{K} \sum_{k=1}^{K} \mathbf{W}^V \hat{\mathbf{H}}^{(k)}_{l,i}.
\end{equation}
This sum is invariant to any permutation of its terms, so $\mathbf{F}'_{i,l}$ is identical regardless of which group element $g$ was applied. Consequently, all views share the same fiber output, establishing $D_4$-invariance.
\end{proof}

\paragraph{Practical implications.}
In practice, the global attention weights are learned rather than fixed at $1/K$. CPGA provides the training signal that drives the encoder toward this invariant fixed point: by enforcing orbit-consistent gradient updates, it encourages the global attention weights to converge to a near-uniform distribution, as corroborated empirically by the 53.6\% reduction in standard deviation under $D_4$ perturbations (Table~\ref{tab:stability}) and the representation-level alignment metrics reported in Section~\ref{sec:repr-align}.

\subsection{Collective Policy Gradient Aggregation}

To explicitly enforce geometric grounding during optimization, we introduce Collective Policy Gradient Aggregation (CPGA). Rather than treating the $K$ symmetric views as independent augmented samples, CPGA is a \emph{symmetry-constrained policy gradient} that processes the full $D_4$ orbit synchronously and enforces orbit-consistent parameter updates, drawing on consensus gradient methods from multi-task learning \cite{yu2020gradient,sener2018multi}. The eight views are isometric transformations with equivalent optimal solutions; they induce different gradient directions due to representational asymmetry, and CPGA regularizes these toward a consensus, driving the encoder toward $D_4$-invariant representations. Each geometric transformation $g_k \in \mathcal{G}$ defines a sub-objective over its induced trajectory space:
\begin{equation}
    \mathcal{J}^{(k)}(\theta) = \mathbb{E}_{\pi \sim p_\theta(\cdot|g_k(\mathbf{X}))} [R(\pi)],
\end{equation}
where $R(\pi) = -L(\pi)$ represents the reward function, defined as the negative total route cost of solution $\pi$ (i.e., negative tour length for TSP and negative total route length for CVRP). Distinct from standard stochastic data augmentation, where transformed samples are treated as independent instances (often introducing variance), CPGA processes the full group orbit synchronously to enforce geometric consistency.

We define the global objective $\mathcal{J}_{\mathcal{G}}(\theta)$ as the joint expectation over the uniform distribution of the symmetry group (specifically the $D_4$ dihedral group):
\begin{equation}
    \mathcal{J}_{\mathcal{G}}(\theta) = \mathbb{E}_{\mathbf{X} \sim \mathcal{D}} \left[ \frac{1}{|\mathcal{G}|} \sum_{g \in \mathcal{G}} \int \pi_\theta(\tau | g(\mathbf{X})) R(\tau) d\tau \right].
\end{equation}
By processing $K=8$ symmetric views within a unified batch, the model constructs $N$ parallel trajectories $\{\pi^{(k,i)}\}_{i=1}^N$ for each $g_k \in \mathcal{G}$. This architecture provides a dense supervisory signal, where the parameter update is driven by the Collective Policy Gradient \cite{williams1992simple,yu2020gradient}:
\begin{equation} \label{eq:collective_gradient}
    \nabla_\theta \mathcal{J}_{\mathcal{G}}(\theta) \approx \frac{1}{K \cdot N} \sum_{k=1}^K \sum_{i=1}^N \Big( R(\pi^{(k,i)}) - b^{(k)} \Big) \cdot \nabla_\theta \log p_\theta(\pi^{(k,i)} | g_k(\mathbf{X})),
\end{equation}
where $b^{(k)} = \frac{1}{N} \sum_{i=1}^N R(\pi^{(k,i)})$ serves as a view-specific baseline. Crucially, this view-specific baseline differentiates CPGA from naive augmentation with a global baseline. By normalizing rewards within each orientation, $b^{(k)}$ prevents the gradient from being dominated by ``easier'' geometric views (which naturally yield higher rewards), ensuring that the optimization focuses on relative improvement across the entire symmetry group. The complete training protocol is formalized in Algorithm~\ref{alg:omni_training} (Appendix~\ref{app:algorithm}).

A formal complexity analysis is provided in Appendix~\ref{app:complexity}. In brief, the MAA encoder matches the $\mathcal{O}(n^2 d)$ asymptotic complexity of the standard single-view Transformer in POMO, and the empirical training overhead is 6\% (5.3 vs.\ 5.0 min per epoch).

\section{Experimental Results}
\label{sec-exp}

\subsection{Experimental Setting}

\paragraph{Datasets}

MViewRouter is evaluated across synthetic and real-world TSP and CVRP instances.

\textit{TSP Instances} \quad Node coordinates are uniformly sampled from $[0, 1]^2$. We train on TSP-50 and TSP-100 and evaluate zero-shot generalization on TSP-200 and TSP-500. Each test set contains 10,000 instances. For out-of-distribution generalization, we additionally evaluate on 29 TSPLIB instances \cite{reinelt1991tsplib} (50--200 nodes) using the TSP-100 model without retraining or fine-tuning.

\textit{CVRP Instances} \quad Following standard practice \cite{kwon2020pomo,queiroga202110}, node coordinates are uniformly sampled from $[0, 1]^2$, customer demands are drawn uniformly from $\{1, \dots, 9\}$, and vehicle capacity is set to $Q=40$ for CVRP-50 and $Q=50$ for CVRP-100. We train separate models on CVRP-50 and CVRP-100 and generate 10,000 instances for each test set.

\paragraph{Baselines}
For TSP, we compare against exact solvers (Concorde \cite{applegate2006concorde}, Gurobi \cite{pedroso2011optimization}), traditional heuristics (OR-Tools, LKH3 \cite{helsgaun2017extension}), and recent state-of-the-art (SOTA) neural solvers: POMO \cite{kwon2020pomo}, Sym-NCO \cite{kim2022sym}, Pointerformer \cite{jin2023pointerformer}, POMO+PO \cite{pan2025preference}, and DIFUSCO \cite{sun2023difusco}. For CVRP, we compare against LKH3 as the heuristic reference and Sym-NCO, POMO, and POMO+PO as neural baselines.

\paragraph{Implementation Details}
Full hyperparameters are provided in Appendix~\ref{app:impl}. To ensure a fair comparison with POMO, the per-view batch size is 8 instances, giving a total of $8\text{ views}\times 8\text{ instances}=64$, identical to POMO's batch size. The $\times$8 views therefore constitute a \emph{reallocation} of the same training budget, not an increase in data volume.


\begin{figure*}[!t]
    \centering
    \includegraphics[width=0.9\textwidth]{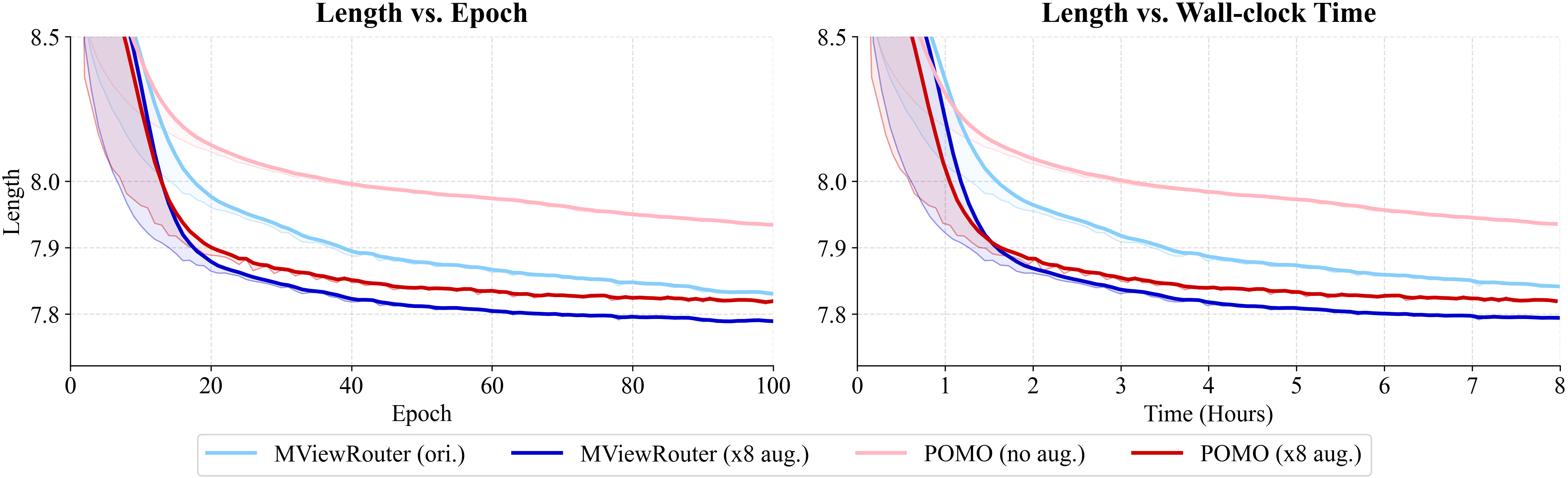}
    \caption{Training curves on TSP-100: POMO and MViewRouter under original view and $\times$8 aug.}
    \label{fig:training_curves}
    \vspace{-1em}
\end{figure*}

\subsection{Training Dynamics}
\label{sec:training_dynamics}

We analyze training dynamics under two evaluation protocols: (1) \textbf{Original View} (decoding from the original, untransformed problem instance), and (2) \textbf{$\times$8 Augmentations}, selecting the best tour across all 8 $D_4$ views.

As shown in Figure~\ref{fig:training_curves}, MViewRouter's original-view performance already matches augmented POMO, while its $\times$8 augmentation curve maintains a consistent lead throughout training. This acceleration is attributed to CPGA: aggregating consensus gradients over the full $D_4$ orbit provides a denser supervisory signal that forces the policy to internalize geometric invariants early, yielding a more robust optimization landscape. An analogous pattern holds on CVRP-100 (Appendix~\ref{app:training_curves}).

\subsection{Comparative Results on Generated Instances}

\paragraph{TSP.}

\begin{wrapfigure}{r}{0.48\textwidth}
\vspace{-1\baselineskip}
    \centering
    \includegraphics[width=\linewidth]{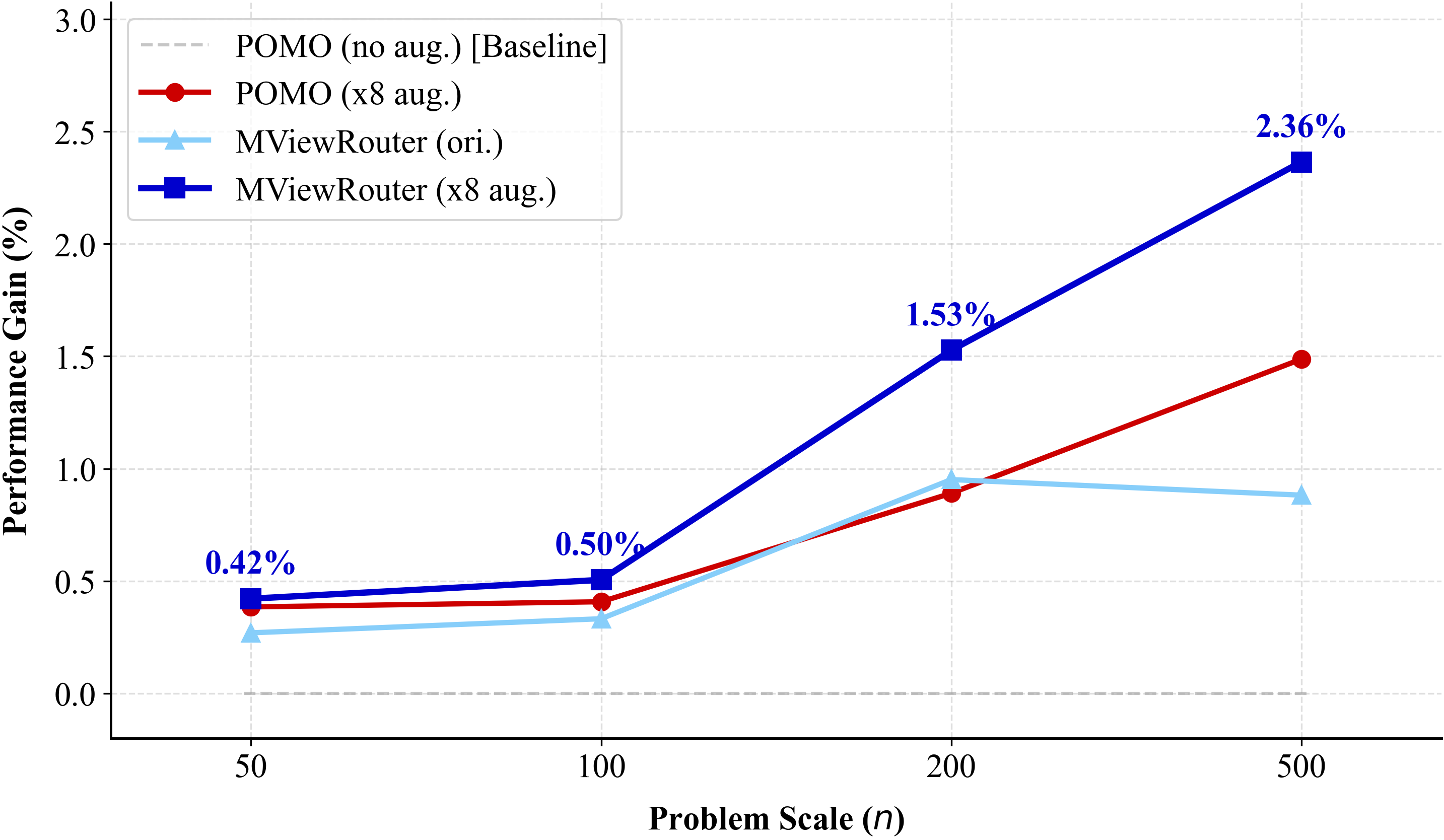}
    \vspace{-0.5em}
    \caption{Zero-shot generalization gap relative to POMO.}
    \label{fig:generalization}
    \vspace{-2\baselineskip}
\end{wrapfigure}

As shown in Table~\ref{tab:tsp_results}, MViewRouter achieves SOTA performance among RL-based neural solvers on both TSP-50 and TSP-100. In the original view, it achieves a 0.25\% gap on TSP-100, rivaling augmented POMO (0.17\%) while running $5.5\times$ faster. With full $\times$8 augmentation, the gap drops to 0.07\%, matching near-exact solver performance. 

To demonstrate generalization, as shown in Figure~\ref{fig:generalization}, this advantage grows with problem scale, outperforming POMO by 0.89\% at $n{=}500$ under $\times$8 augmentation, while the original-view model already surpasses augmented POMO at $n{=}200$.

\vspace{0.5em}
\paragraph{CVRP.}

\begin{table}[!t]
\centering
\renewcommand{\arraystretch}{1.1} 
\caption{Results on TSP-50 and TSP-100. Gap relative to Concorde. Type: Reinforcement Learning (RL) / Supervised Learning (SL) / Preference Optimization (PO) = training paradigm; Greedy (G) / Sampling (S) = decoding mode. Bold: best neural solver.}
\label{tab:tsp_results}
\resizebox{0.82\columnwidth}{!}{
\begin{tabular}{l|l l ccc ccc}
\toprule
& \multirow{2}{*}{\textbf{Method}} & \multirow{2}{*}{\textbf{Type}} & \multicolumn{3}{c}{\textbf{TSP-50}} & \multicolumn{3}{c}{\textbf{TSP-100}} \\
\cmidrule(lr){4-6} \cmidrule(lr){7-9}
& & & Len. & Gap & Time & Len. & Gap & Time \\
\midrule
\multirow{4}{*}{Traditional}
& Concorde  & Exact      & 5.69 & 0.00\% & (13m) & 7.76 & 0.00\% & (1h) \\
& Gurobi    & Exact      & 5.69 & 0.00\% & (2m)  & 7.76 & 0.00\% & (17m) \\
& LKH3      & Heuristic  & 5.69 & 0.00\% & (6m)  & 7.76 & 0.00\% & (25m) \\
& OR Tools  & Heuristic  & 5.85 & 2.87\% & (5m)  & 8.06 & 3.86\% & (23m) \\
\midrule
\multirow{9}{*}{Neural}
& POMO, no aug.              & RL+S  & 5.70 & 0.12\% & (1s)  & 7.80 & 0.58\% & (6s) \\
& POMO, $\times$8 aug.       & RL+S  & 5.69 & 0.03\% & (9s)  & 7.77 & 0.17\% & (44s) \\
& Sym-NCO                    & RL+S  & 5.70 & 0.16\% & (9s)  & 7.79 & 0.39\% & (44s) \\
& Pointerformer              & RL    & 5.69 & 0.02\% & (12s) & 7.77 & 0.15\% & (1m) \\
& DIFUSCO                    & SL+G  & 5.71 & 0.45\% & (9m)  & 7.85 & 1.21\% & (9m) \\
& DIFUSCO                    & SL+S  & 5.70 & 0.09\% & —     & 7.78 & 0.23\% & — \\
& POMO+PO                    & PO+S  & 5.69 & 0.02\% & (9s)  & 7.76 & 0.07\% & (44s) \\
\cmidrule{2-9}
& \cellcolor[gray]{0.9}\textbf{MViewRouter, ori.}     & \cellcolor[gray]{0.9}RL+S  & \cellcolor[gray]{0.9}5.70 & \cellcolor[gray]{0.9}0.08\% & \cellcolor[gray]{0.9}(2s)  & \cellcolor[gray]{0.9}7.78 & \cellcolor[gray]{0.9}0.25\% & \cellcolor[gray]{0.9}(8s) \\
& \cellcolor[gray]{0.9}\textbf{MViewRouter, $\times$8 aug.} & \cellcolor[gray]{0.9}RL+S & \cellcolor[gray]{0.9}\textbf{5.69} & \cellcolor[gray]{0.9}\textbf{0.01\%} & \cellcolor[gray]{0.9}(12s) & \cellcolor[gray]{0.9}\textbf{7.76} & \cellcolor[gray]{0.9}\textbf{0.07\%} & \cellcolor[gray]{0.9}(1m) \\
\bottomrule
\end{tabular}
}
\vspace{-0.5em}
\end{table}

\begin{table}[!t]
\centering
\renewcommand{\arraystretch}{1.1}
\caption{Results on CVRP-50 and CVRP-100. Gap relative to LKH3.}
\label{tab:cvrp_results}
\resizebox{0.82\columnwidth}{!}{
\begin{tabular}{l|l ccc ccc}
\toprule
& \multirow{2}{*}{\textbf{Method}} & \multicolumn{3}{c}{\textbf{CVRP-50}} & \multicolumn{3}{c}{\textbf{CVRP-100}} \\
\cmidrule(lr){3-5} \cmidrule(lr){6-8}
& & Len. & Gap & Time & Len. & Gap & Time \\
\midrule
Traditional & LKH3    & 10.38 & 0.00\% & (7h)  & 15.68 & 0.00\% & (12h) \\
\midrule
\multirow{6}{*}{Neural}
& POMO, no aug.           & 10.49 & 1.14\% & (1s)  & 15.83 & 0.98\% & (5s) \\
& POMO, $\times$8 aug.         & 10.46 & 0.76\% & (8s)  & 15.73 & 0.32\% & (40s) \\
& Sym-NCO      & 10.50 & 1.13\% & (8s)  & 15.92 & 1.51\% & (40s) \\
& POMO+PO      & 10.43 & 0.50\% & (8s)  & 15.81 & 0.82\% & (40s) \\
\cmidrule{2-8}
& \cellcolor[gray]{0.9}\textbf{MViewRouter, ori.}           & \cellcolor[gray]{0.9}10.45 & \cellcolor[gray]{0.9}0.67\% & \cellcolor[gray]{0.9}(2s)  & \cellcolor[gray]{0.9}15.81 & \cellcolor[gray]{0.9}0.82\% & \cellcolor[gray]{0.9}(7s)  \\
& \cellcolor[gray]{0.9}\textbf{MViewRouter, $\times$8 aug.} & \cellcolor[gray]{0.9}\textbf{10.40} & \cellcolor[gray]{0.9}\textbf{0.19\%} & \cellcolor[gray]{0.9}(12s) & \cellcolor[gray]{0.9}\textbf{15.71} & \cellcolor[gray]{0.9}\textbf{0.19\%} & \cellcolor[gray]{0.9}(55s) \\
\bottomrule
\end{tabular}
}
\vspace{-1em}
\end{table}

As shown in Table~\ref{tab:cvrp_results}, MViewRouter consistently outperforms all neural baselines on both problem sizes under both evaluation protocols, demonstrating that MAA and CPGA transfer naturally to CVRP without architectural modification. With $\times$8 augmentation, it achieves 0.19\% on both CVRP-50 and CVRP-100, outperforming POMO+PO by 0.31\% and 0.63\% respectively. In the original-view setting, it matches POMO+PO on CVRP-100 at 0.82\% while running $5.7\times$ faster, confirming the favorable accuracy--efficiency trade-off of internalized equivariance.

\subsection{Comparative Results on Real-world Instances}

\begin{table}[!t]
\centering
\renewcommand{\arraystretch}{1.1}
\begin{minipage}[t]{0.48\textwidth}
\centering
\caption{Mean results on 29 TSPLIB instances ($\times$8 aug.). Full per-instance results in Appendix~\ref{app:tsplib}.}
\label{tab:TSPLIB_summary}
\resizebox{0.89\textwidth}{!}{
\begin{tabular}{l ccc}
\toprule
\textbf{Method} & \textbf{Mean Len.} & \textbf{Mean Gap (\%)} & \textbf{Time (s)} \\
\midrule
POMO        & 6.683 & 2.373 & 0.043 \\
\rowcolor[gray]{0.9}
\textbf{MViewRouter} & \textbf{6.650} & \textbf{1.769} & 0.052 \\
\bottomrule
\end{tabular}
}
\end{minipage}
\hfill
\begin{minipage}[t]{0.48\textwidth}
\centering
\caption{Inference stability under $D_4$ perturbations on 10{,}000 TSP-100 instances. Gaps relative to best view; $\sigma$ = std.\ dev.\ of tour lengths.}
\label{tab:stability}
\resizebox{\textwidth}{!}{
\begin{tabular}{l ccc}
\toprule
\textbf{Method} & \textbf{Mean Gap (\%)} & \textbf{Max Gap (\%)} & \textbf{Std. Dev. ($\sigma$)} \\
\midrule
POMO & 0.41 & 1.00 & 0.028 \\
\rowcolor[gray]{0.9}
\textbf{MViewRouter} & \textbf{0.17} & \textbf{0.42} & \textbf{0.012} \\
\bottomrule
\end{tabular}
}
\end{minipage}
\vspace{-0.5em}
\end{table}

As shown in Table~\ref{tab:TSPLIB_summary}, MViewRouter achieves a 25.5\% relative reduction in the average optimality gap over POMO on 29 TSPLIB instances, with full per-instance results in Appendix~\ref{app:tsplib}.

\subsection{Geometric Invariance and Inference Stability}
\label{sec:stability}

\begin{wrapfigure}{r}{0.45\textwidth}
    \centering
    \includegraphics[width=1\linewidth]{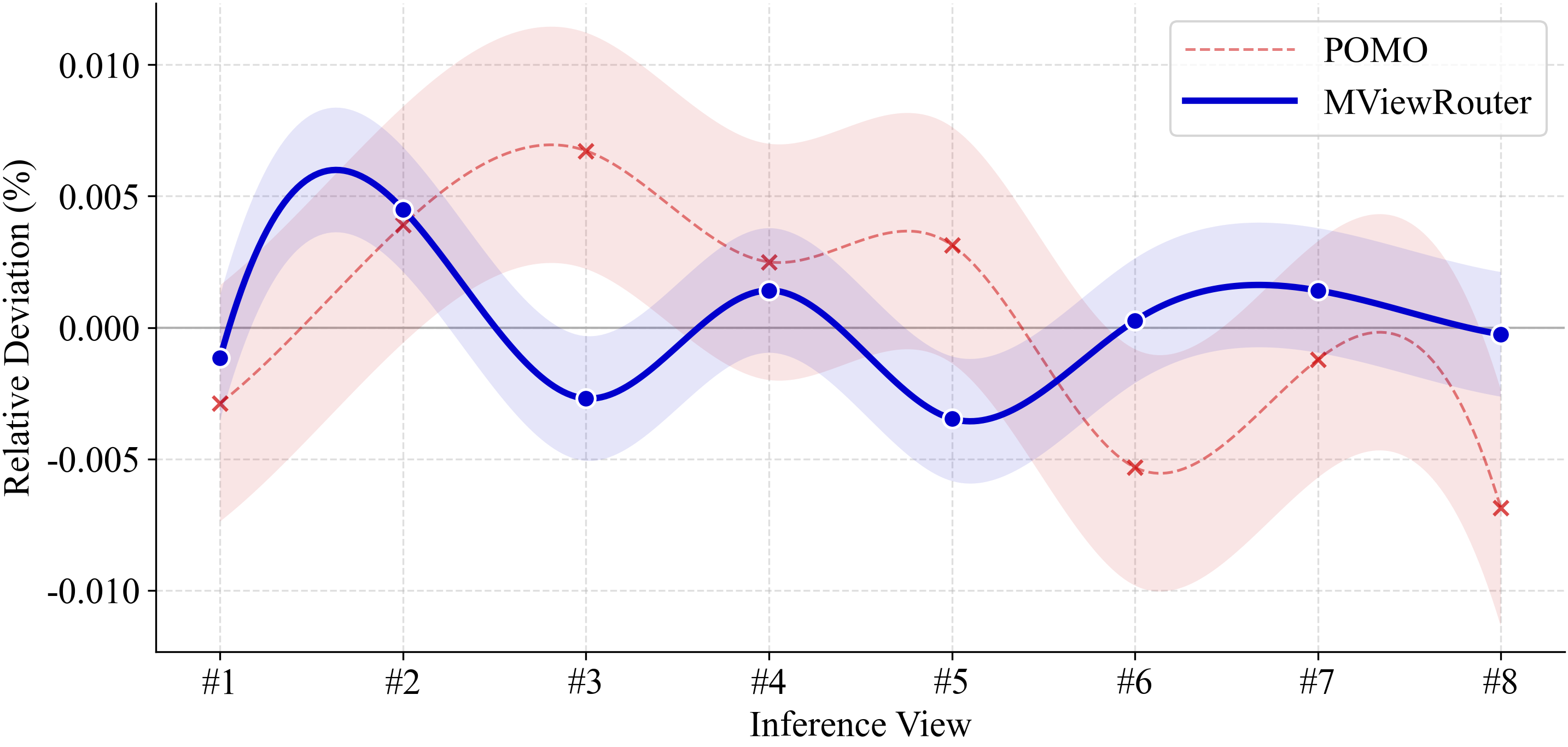}
    \caption{Per-view tour lengths across the $D_4$ symmetry group on TSP-100.}
    \label{fig:consistency_plot}
    \vspace{-\baselineskip}
\end{wrapfigure}

We apply all 8 $D_4$ transformations to 10,000 unseen TSP-100 instances and measure variance across
orientations as a proxy for internalized geometric invariance. As shown in Table~\ref{tab:stability} and Figure~\ref{fig:consistency_plot}, MViewRouter reduces $\sigma$ by 53.6\% over POMO.

While POMO's performance fluctuates with coordinate orientation, MViewRouter shows a nearly flat trajectory. The 53.6\% $\sigma$ reduction is the inference-time manifestation of the representation-level alignment that MAA achieves during training (Section~\ref{sec:ablation}), converging toward the $D_4$-invariant fixed point of Proposition~\ref{prop:maa_invariance}.

\subsection{Impact of Training View Quantity}
\label{sec:view_quantity}

We investigate the influence of the training view quantity $K \in \{1, 2, 4, 8\}$ on the convergence of MViewRouter, examining both solution quality and encoder-level geometric consistency.

\begin{figure*}[!h]
    \centering
    \includegraphics[width=1\textwidth]{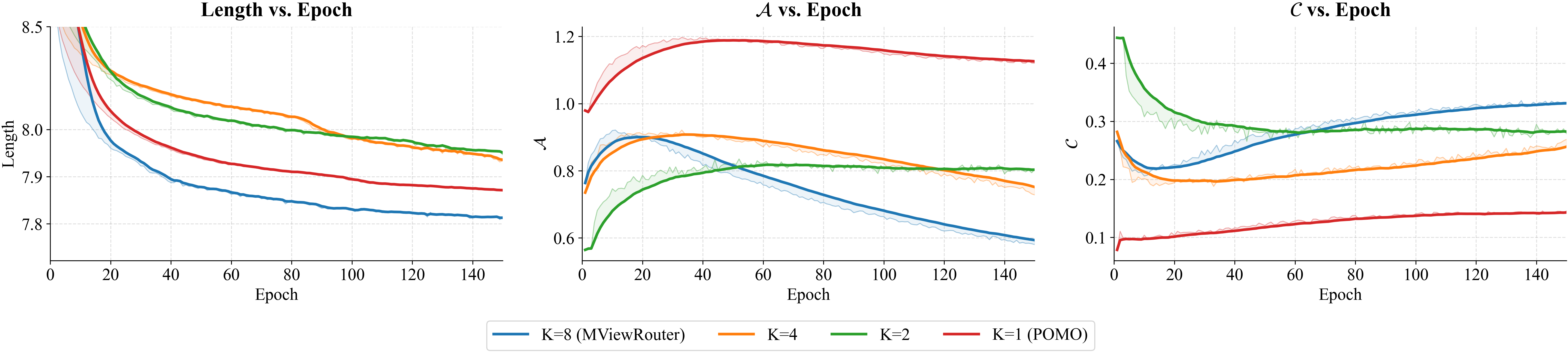}
    \vspace{-1em}
    \caption{Impact of view quantity $K \in \{1,2,4,8\}$ on TSP-100. Left: solution quality; center: alignment gap $\mathcal{A}$; right: cosine similarity $\mathcal{C}$.}
    \label{fig:view_analysis}
    \vspace{-1em}
\end{figure*}

\paragraph{Solution quality (Figure~\ref{fig:view_analysis}, left).}
The solution quality follows a non-monotonic pattern in $K$: $K=8$ achieves the best final gap via full $D_4$ orbit coverage, while $K=1$ outperforms intermediate $K=2,4$ at convergence since it incurs no alignment overhead. Early in training, $K=2$ temporarily leads $K=4$: the latter introduces $6\times$ more cross-view pairs without the stabilizing effect of full group closure, creating an alignment overhead that resolves once the encoder catches up. The $K=8$ configuration avoids this slow-start entirely, confirming that group completeness is the key structural property for rapid, stable convergence.
\paragraph{Representation alignment (Figure~\ref{fig:view_analysis}, center and right).}
\label{sec:repr-align}
To directly measure whether the encoder has learned geometrically consistent representations, we track two metrics on the final encoder layer. Let $\mathbf{H}^{(k)}_i$ denote the output embedding for node $i$ under view $k$. The cross-view alignment gap $\mathcal{A}$ measures the mean intra-node representation variance (lower is better); the cross-view cosine similarity $\mathcal{C}$~\cite{radford2021learning} measures their angular agreement (higher is better):
\begin{equation}
    \mathcal{A} = \frac{1}{n}\sum_{i=1}^{n}\frac{1}{K}\sum_{k=1}^{K}\left\|\mathbf{H}^{(k)}_i - \bar{\mathbf{H}}_i\right\|^2, \quad \bar{\mathbf{H}}_i = \frac{1}{K}\sum_{k=1}^{K}\mathbf{H}^{(k)}_i, \label{eq:align_gap}
\end{equation}
\begin{equation}
    \mathcal{C} = \frac{2}{nK(K-1)}\sum_{i=1}^{n}\sum_{j<k}\frac{\mathbf{H}^{(j)}_i \cdot \mathbf{H}^{(k)}_i}{\|\mathbf{H}^{(j)}_i\|\,\|\mathbf{H}^{(k)}_i\|}\quad (K > 1). \label{eq:cosine_sim}
\end{equation}
For POMO ($K=1$), both metrics are evaluated by encoding 8 augmentations independently through its encoder. Both metrics tell a consistent story. In early training, $K=4$ lags $K=2$ in both $\mathcal{A}$ and $\mathcal{C}$, confirming that the slower convergence is caused by unresolved alignment complexity rather than model capacity. Once this alignment difficulty is overcome, $K=4$ surpasses $K=2$ in both metrics. $K=8$ dominates monotonically throughout, as full $D_4$ closure provides a coherent gradient consensus from epoch one. Together, the curves establish a causal chain: complete group coverage, faster alignment, better optimization.


\section{Ablation Study}
\label{sec:ablation}

To evaluate the contribution of each architectural component in MViewRouter, we conduct a systematic ablation study on the TSP-100 benchmark. We analyze training convergence against two degraded variants: (1) \textbf{w/o CPGA}, which uses MAA but updates parameters with only the original-view gradient; and (2) \textbf{w/o MAA}, which processes eight views independently, skipping the global cross-view fiber attention.

\begin{figure*}[!h]
    \centering
    \includegraphics[width=1\textwidth]{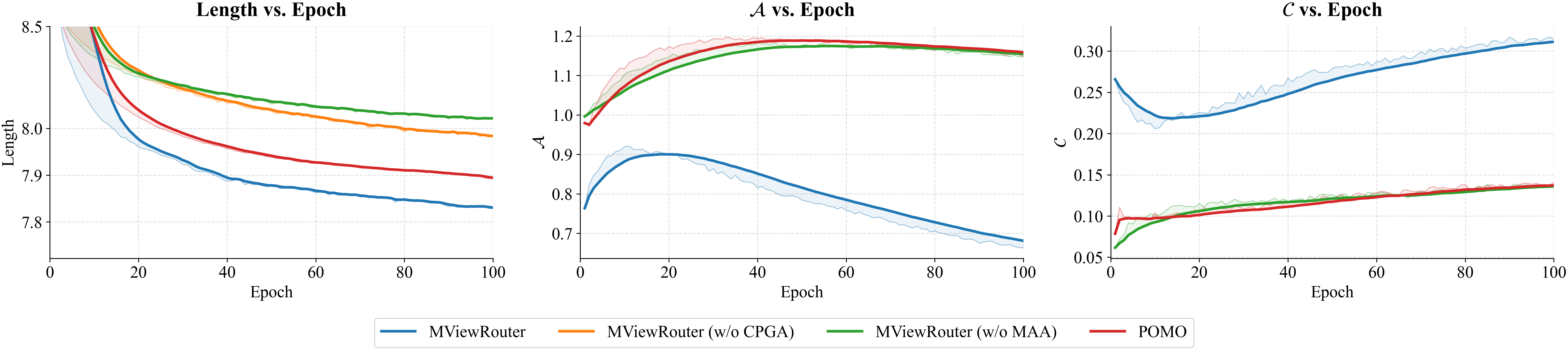}
    \vspace{-0.5em}
    \caption{Ablation on TSP-100. Left: solution quality; center: alignment gap $\mathcal{A}$; right: cosine similarity $\mathcal{C}$.}
    \label{fig:ablation}
\end{figure*}

\paragraph{Solution quality (Figure~\ref{fig:ablation}, left).}
CPGA mainly improves optimization efficiency: aggregating gradients over the $D_4$ orbit provides denser and more stable supervision, allowing the full model to exceed the ``w/o CPGA'' variant's final performance in less than half the training time. MAA, conversely, is critical for final optimality: the ``w/o MAA'' variant stagnates at a much higher optimality gap despite retaining full multi-view training, confirming that inter-view fiber attention (not merely multi-view gradient aggregation) is the essential mechanism for geometric grounding.

\paragraph{Representation alignment (Figure~\ref{fig:ablation}, center and right).}
POMO \cite{kwon2020pomo} and w/o MAA remain nearly indistinguishable in both $\mathcal{A}$ and $\mathcal{C}$ throughout training, confirming that gradient aggregation alone cannot drive the encoder toward $D_4$-invariant representations. The full MViewRouter achieves substantially lower $\mathcal{A}$ ($0.920 \to 0.664$) and higher $\mathcal{C}$ ($0.206 \to 0.314$), respectively, confirming that MAA's inter-view fiber attention is the essential mechanism for geometric invariance.





\section{Conclusion}
We introduced MViewRouter, a problem-agnostic framework that internalizes geometric equivariance for combinatorial routing. By integrating Multi-view Alternating Attention and Collective Policy Gradient Aggregation, our model transforms geometric symmetry from a post-hoc augmentation trick into a core structural inductive bias. Experiments on TSP, CVRP, and real-world TSPLIB benchmarks confirm that MViewRouter achieves superior solution quality, robust zero-shot generalization, and strong inference stability. Crucially, by rivaling augmented-baseline performance in the original view and generalizing naturally to CVRP without architectural modification, MViewRouter demonstrates the effectiveness of internalizing geometric equivariance as a universal structural prior for Euclidean routing problems.

\paragraph{Limitations.}
Scaling to instances with $1\text{k}$--$10\text{k}$ nodes remains an open challenge shared by all autoregressive constructive solvers \cite{kwon2020pomo,kim2022sym,jin2023pointerformer,pan2025preference}, and is orthogonal to our focus on geometric equivariance. Extending MViewRouter to further routing variants (e.g., VRPTW, orienteering) and combining it with fine-tuning \cite{pan2025preference} or active search \cite{hottung2021efficient,choo2022simulation} are natural directions for future work.

\bibliographystyle{plainnat}
\bibliography{example_paper}

@String{Computing = "Computing" }

@String{Computer = "{IEEE} Computer" }

@String{Chelsea = "Chelsea" }

@String{Springer = "Springer-Verlag" }

@inproceedings{queiroga202110,
  title={10,000 optimal CVRP solutions for testing machine learning based heuristics},
  author={Queiroga, Eduardo and Sadykov, Ruslan and Uchoa, Eduardo and Vidal, Thibaut},
  booktitle={AAAI-22 workshop on machine learning for operations research (ML4OR)},
  year={2021}
}

@inproceedings{pan2025preference,
  title={Preference Optimization for Combinatorial Optimization Problems},
  author={Pan, Mingjun and Lin, Guanquan and Luo, You-Wei and Zhu, Bin and Dai, Zhien and Sun, Lijun and Yuan, Chun},
  booktitle={Forty-second International Conference on Machine Learning},
  year={2025}
}

@article{liao2025bopo,
  title={BOPO: Neural combinatorial optimization via best-anchored and objective-guided preference optimization},
  author={Liao, Zijun and Chen, Jinbiao and Wang, Debing and Zhang, Zizhen and Wang, Jiahai},
  journal={arXiv preprint arXiv:2503.07580},
  year={2025}
}

@article{ma2024learning,
  title={Learning to search feasible and infeasible regions of routing problems with flexible neural k-opt},
  author={Ma, Yining and Cao, Zhiguang and Chee, Yeow Meng},
  journal={Advances in Neural Information Processing Systems},
  volume={36},
  pages={49555--49578},
  year={2023}
}

@article{hottung2021efficient,
  title={Efficient active search for combinatorial optimization problems},
  author={Hottung, Andr{\'e} and Kwon, Yeong-Dae and Tierney, Kevin},
  journal={arXiv preprint arXiv:2106.05126},
  year={2021}
}

@article{choo2022simulation,
  title={Simulation-guided beam search for neural combinatorial optimization},
  author={Choo, Jinho and Kwon, Yeong-Dae and Kim, Jihoon and Jae, Jeongwoo and Hottung, Andr{\'e} and Tierney, Kevin and Gwon, Youngjune},
  journal={Advances in Neural Information Processing Systems},
  volume={35},
  pages={8760--8772},
  year={2022}
}

@inproceedings{zheng2021combining,
  title={Combining reinforcement learning with Lin-Kernighan-Helsgaun algorithm for the traveling salesman problem},
  author={Zheng, Jiongzhi and He, Kun and Zhou, Jianrong and Jin, Yan and Li, Chu-Min},
  booktitle={Proceedings of the AAAI conference on artificial intelligence},
  volume={35},
  number={14},
  pages={12445--12452},
  year={2021}
}

@article{xin2021neurolkh,
  title={Neurolkh: Combining deep learning model with lin-kernighan-helsgaun heuristic for solving the traveling salesman problem},
  author={Xin, Liang and Song, Wen and Cao, Zhiguang and Zhang, Jie},
  journal={Advances in Neural Information Processing Systems},
  volume={34},
  pages={7472--7483},
  year={2021}
}

@article{hottung2020neural,
  title={Neural large neighborhood search for the capacitated vehicle routing problem},
  author={Hottung, Andr{\'e} and Tierney, Kevin},
  journal={arXiv preprint arXiv:1911.09539},
  year={2019}
}

@inproceedings{lu2019learning,
  title={A learning-based iterative method for solving vehicle routing problems},
  author={Lu, Hao and Zhang, Xingwen and Yang, Shuang},
  booktitle={International conference on learning representations},
  year={2019}
}

@article{grinsztajn2024winner,
  title={Winner takes it all: Training performant rl populations for combinatorial optimization},
  author={Grinsztajn, Nathan and Furelos-Blanco, Daniel and Surana, Shikha and Bonnet, Cl{\'e}ment and Barrett, Tom},
  journal={Advances in Neural Information Processing Systems},
  volume={36},
  pages={48485--48509},
  year={2023}
}

@article{kim2022sym,
  title={Sym-nco: Leveraging symmetricity for neural combinatorial optimization},
  author={Kim, Minsu and Park, Junyoung and Park, Jinkyoo},
  journal={Advances in Neural Information Processing Systems},
  volume={35},
  pages={1936--1949},
  year={2022}
}

@inproceedings{xin2021multi,
  title={Multi-decoder attention model with embedding glimpse for solving vehicle routing problems},
  author={Xin, Liang and Song, Wen and Cao, Zhiguang and Zhang, Jie},
  booktitle={Proceedings of the AAAI conference on artificial intelligence},
  volume={35},
  number={13},
  pages={12042--12049},
  year={2021}
}

@article{nazari2018reinforcement,
  title={Reinforcement learning for solving the vehicle routing problem},
  author={Nazari, Mohammadreza and Oroojlooy, Afshin and Snyder, Lawrence and Tak{\'a}c, Martin},
  journal={Advances in neural information processing systems},
  volume={31},
  year={2018}
}

@article{vinyals2015pointer,
  title={Pointer networks},
  author={Vinyals, Oriol and Fortunato, Meire and Jaitly, Navdeep},
  journal={Advances in neural information processing systems},
  volume={28},
  year={2015}
}

@article{baldacci2010some,
  title={Some applications of the generalized vehicle routing problem},
  author={Baldacci, Roberto and Bartolini, Enrico and Laporte, Gilbert},
  journal={Journal of the operational research society},
  volume={61},
  number={7},
  pages={1072--1077},
  year={2010},
  publisher={Taylor \& Francis}
}

@inproceedings{pan2023h,
  title={H-tsp: Hierarchically solving the large-scale traveling salesman problem},
  author={Pan, Xuanhao and Jin, Yan and Ding, Yuandong and Feng, Mingxiao and Zhao, Li and Song, Lei and Bian, Jiang},
  booktitle={Proceedings of the AAAI Conference on Artificial Intelligence},
  volume={37},
  number={8},
  pages={9345--9353},
  year={2023}
}

@article{kool2018attention,
  title={Attention, learn to solve routing problems!},
  author={Kool, Wouter and Van Hoof, Herke and Welling, Max},
  journal={arXiv preprint arXiv:1803.08475},
  year={2018}
}

@article{kwon2020pomo,
  title={Pomo: Policy optimization with multiple optima for reinforcement learning},
  author={Kwon, Yeong-Dae and Choo, Jinho and Kim, Byoungjip and Yoon, Iljoo and Gwon, Youngjune and Min, Seungjai},
  journal={Advances in Neural Information Processing Systems},
  volume={33},
  pages={21188--21198},
  year={2020}
}

@inproceedings{jin2023pointerformer,
  title={Pointerformer: Deep reinforced multi-pointer transformer for the traveling salesman problem},
  author={Jin, Yan and Ding, Yuandong and Pan, Xuanhao and He, Kun and Zhao, Li and Qin, Tao and Song, Lei and Bian, Jiang},
  booktitle={Proceedings of the AAAI Conference on Artificial Intelligence},
  volume={37},
  number={7},
  pages={8132--8140},
  year={2023}
}

@misc{applegate2006concorde,
  title={Concorde TSP solver},
  author={Applegate, David and Bixby, Ribert and Chvatal, Vasek and Cook, William},
  year={2006}
}

@article{helsgaun2017extension,
  title={An extension of the Lin-Kernighan-Helsgaun TSP solver for constrained traveling salesman and vehicle routing problems},
  author={Helsgaun, Keld},
  journal={Roskilde: Roskilde University},
  volume={12},
  pages={966--980},
  year={2017}
}

@article{bengio2021machine,
  title={Machine learning for combinatorial optimization: a methodological tour d’horizon},
  author={Bengio, Yoshua and Lodi, Andrea and Prouvost, Antoine},
  journal={European Journal of Operational Research},
  volume={290},
  number={2},
  pages={405--421},
  year={2021},
  publisher={Elsevier}
}

@article{vaswani2017attention,
  title={Attention is all you need},
  author={Vaswani, Ashish and Shazeer, Noam and Parmar, Niki and Uszkoreit, Jakob and Jones, Llion and Gomez, Aidan N and Kaiser, {\L}ukasz and Polosukhin, Illia},
  journal={Advances in neural information processing systems},
  volume={30},
  year={2017}
}

@article{bello2016neural,
  title={Neural combinatorial optimization with reinforcement learning},
  author={Bello, Irwan and Pham, Hieu and Le, Quoc V and Norouzi, Mohammad and Bengio, Samy},
  journal={arXiv preprint arXiv:1611.09940},
  year={2016}
}

@article{sun2023difusco,
  title={Difusco: Graph-based diffusion solvers for combinatorial optimization},
  author={Sun, Zhiqing and Yang, Yiming},
  journal={Advances in neural information processing systems},
  volume={36},
  pages={3706--3731},
  year={2023}
}

@article{wu2021learning,
  title={Learning improvement heuristics for solving routing problems},
  author={Wu, Yaoxin and Song, Wen and Cao, Zhiguang and Zhang, Jie and Lim, Andrew},
  journal={IEEE transactions on neural networks and learning systems},
  volume={33},
  number={9},
  pages={5057--5069},
  year={2021},
  publisher={IEEE}
}

@article{reinelt1991tsplib,
  title={TSPLIB—A traveling salesman problem library},
  author={Reinelt, Gerhard},
  journal={ORSA journal on computing},
  volume={3},
  number={4},
  pages={376--384},
  year={1991},
  publisher={INFORMS}
}

@inproceedings{fang2024invit,
  title={INViT: A Generalizable Routing Problem Solver with Invariant Nested View Transformer},
  author={Fang, Han and Song, Zhihao and Weng, Paul and Ban, Yutong},
  booktitle={International Conference on Machine Learning},
  pages={12973--12992},
  year={2024},
  organization={PMLR}
}

@inproceedings{wang2025vggt,
  title={Vggt: Visual geometry grounded transformer},
  author={Wang, Jianyuan and Chen, Minghao and Karaev, Nikita and Vedaldi, Andrea and Rupprecht, Christian and Novotny, David},
  booktitle={Proceedings of the Computer Vision and Pattern Recognition Conference},
  pages={5294--5306},
  year={2025}
}

@article{pedroso2011optimization,
  title={Optimization with gurobi and python},
  author={Pedroso, Joo Pedro},
  journal={INESC Porto and Universidade do Porto,, Porto, Portugal},
  volume={1},
  year={2011}
}

@inproceedings{satorras2021n,
  title={E (n) equivariant graph neural networks},
  author={Satorras, V{\i}ctor Garcia and Hoogeboom, Emiel and Welling, Max},
  booktitle={International conference on machine learning},
  pages={9323--9332},
  year={2021},
  organization={PMLR}
}

@inproceedings{arnab2021vivit,
  title={Vivit: A video vision transformer},
  author={Arnab, Anurag and Dehghani, Mostafa and Heigold, Georg and Sun, Chen and Lu{\v{c}}i{\'c}, Mario and Schmid, Cordelia},
  booktitle={Proceedings of the IEEE/CVF international conference on computer vision},
  pages={6836--6846},
  year={2021}
}

@article{bronstein2021geometric,
  title={Geometric deep learning: Grids, groups, graphs, geodesics, and gauges},
  author={Bronstein, Michael M and Bruna, Joan and Cohen, Taco and Veli{\v{c}}kovi{\'c}, Petar},
  journal={arXiv preprint arXiv:2104.13478},
  year={2021}
}

@article{yu2020gradient,
  title={Gradient surgery for multi-task learning},
  author={Yu, Tianhe and Kumar, Saurabh and Gupta, Abhishek and Levine, Sergey and Hausman, Karol and Finn, Chelsea},
  journal={Advances in neural information processing systems},
  volume={33},
  pages={5824--5836},
  year={2020}
}

@article{sener2018multi,
  title={Multi-task learning as multi-objective optimization},
  author={Sener, Ozan and Koltun, Vladlen},
  journal={Advances in neural information processing systems},
  volume={31},
  year={2018}
}

@article{williams1992simple,
  title={Simple statistical gradient-following algorithms for connectionist reinforcement learning},
  author={Williams, Ronald J},
  journal={Machine learning},
  volume={8},
  number={3},
  pages={229--256},
  year={1992},
  publisher={Springer}
}

@article{kingma2014adam,
  title={Adam: A method for stochastic optimization},
  author={Kingma, Diederik P and Ba, Jimmy},
  journal={arXiv preprint arXiv:1412.6980},
  year={2014}
}

@inproceedings{radford2021learning,
  title={Learning transferable visual models from natural language supervision},
  author={Radford, Alec and Kim, Jong Wook and Hallacy, Chris and Ramesh, Aditya and Goh, Gabriel and Agarwal, Sandhini and Sastry, Girish and Askell, Amanda and Mishkin, Pamela and Clark, Jack and others},
  booktitle={International conference on machine learning},
  pages={8748--8763},
  year={2021},
  organization={PmLR}
}

@article{drakulic2023bq,
  title={Bq-nco: Bisimulation quotienting for efficient neural combinatorial optimization},
  author={Drakulic, Darko and Michel, Sofia and Mai, Florian and Sors, Arnaud and Andreoli, Jean-Marc},
  journal={Advances in Neural Information Processing Systems},
  volume={36},
  pages={77416--77429},
  year={2023}
}

\newpage
\appendix

\section{$D_4$ Symmetry Group Transformations}
\label{app:d4}
The $D_4$ dihedral group consists of 8 isometric transformations of the plane, applied pointwise to node coordinates $(x, y) \in [0,1]^2$: four rotations ($0^\circ$, $90^\circ$, $180^\circ$, $270^\circ$ clockwise) and four reflections (horizontal, vertical, and the two diagonals). Concretely, the 8 transformations $\{g_k\}_{k=1}^{8}$ are listed in Table~\ref{tab:d4}. All transformations preserve pairwise Euclidean distances, ensuring that the optimal tour length is invariant across all 8 views.

\begin{table}[h]
\centering
\caption{The 8 elements of the $D_4$ dihedral group applied pointwise to node coordinates $(x, y) \in [0,1]^2$.}
\label{tab:d4}
\begin{tabular}{cll}
\toprule
$k$ & Transformation & $(x, y) \mapsto$ \\
\midrule
1 & Identity & $(x, y)$ \\
2 & $90^\circ$ rotation & $(y, 1-x)$ \\
3 & $180^\circ$ rotation & $(1-x, 1-y)$ \\
4 & $270^\circ$ rotation & $(1-y, x)$ \\
5 & Horizontal flip & $(1-x, y)$ \\
6 & Vertical flip & $(x, 1-y)$ \\
7 & Diagonal flip & $(y, x)$ \\
8 & Anti-diagonal flip & $(1-y, 1-x)$ \\
\bottomrule
\end{tabular}
\end{table}

\section{Complexity Analysis}
\label{app:complexity}

\paragraph{MAA encoder.}
Each encoder block performs two attention passes. The local pass applies $\mathrm{MHA}_{\mathrm{local}}$ to each of the $K$ views independently over $n$ nodes, costing $\mathcal{O}(K n^2 d)$ per layer. The global pass applies $\mathrm{MHA}_{\mathrm{global}}$ to each of the $n$ node fibers over $K$ views, costing $\mathcal{O}(n K^2 d)$ per layer. The total per-layer complexity is therefore
\begin{equation}
    \mathcal{O}(Kn^2d + nK^2d) = \mathcal{O}(Knd(n + K)).
\end{equation}
Since $K = 8$ is a small fixed constant, this simplifies to $\mathcal{O}(n^2 d)$, identical asymptotic complexity to the standard single-view Transformer \cite{vaswani2017attention} in POMO \cite{kwon2020pomo}. Over $L$ layers, the encoder runs in $\mathcal{O}(L n^2 d)$.

\paragraph{CPGA training.}
All $K$ views are packed into a single unified batch of size $K \times B$ (where $B$ is the view-specific batch size), so no additional forward passes are required beyond what the batch accommodates in parallel. The view-specific baseline and gradient aggregation add $\mathcal{O}(KN)$ scalar operations per step, which is negligible. Empirically, the total training overhead is 6\% over POMO (5.3 vs.\ 5.0 min per epoch), confirming that the theoretical $\mathcal{O}(n^2 d)$ parity translates to near-identical wall-clock cost in practice.

\section{Training Algorithm}
\label{app:algorithm}

Each training iteration processes a mini-batch of instances $\mathbf{X}$ by generating all $K{=}8$ symmetric views in parallel. For each view $g_k(\mathbf{X})$, the policy samples $N$ trajectories and computes a view-specific baseline $b^{(k)}$ as their mean reward. The view-specific baseline is the key design choice that prevents geometrically easier orientations from dominating the gradient signal: by normalizing rewards within each view independently, CPGA ensures that every group element contributes equally to the parameter update regardless of its intrinsic difficulty. The $K$ view-specific gradients are then averaged into a single consensus update, which backpropagates orbit-consistent supervision through both the MAA encoder and the decoder in one step. This is strictly more efficient than running $K$ sequential updates, and produces a lower-variance gradient estimate than treating the views as independent augmented samples with a shared baseline.

\begin{algorithm}[h]
    \caption{Training of MViewRouter with CPGA}
    \label{alg:omni_training}
    \begin{algorithmic}[1]
        \STATE {\bfseries Input:} Dataset $\mathcal{D}$, Symmetry Group $\mathcal{G}$, Learning rate $\eta$
        \STATE {\bfseries Initialize:} Policy parameters $\theta$
        \FOR{each training iteration}
            \STATE Sample mini-batch of instances $\mathbf{X} \sim \mathcal{D}$
            \FOR{each group element $g_k \in \mathcal{G}$ \textbf{in parallel}}
                \STATE $\mathbf{X}^{(k)} \leftarrow g_k(\mathbf{X})$
                \STATE Sample $N$ trajectories: $\{\pi^{(k,i)}\}_{i=1}^N \sim p_\theta(\cdot | \mathbf{X}^{(k)})$
                \STATE Compute view-specific baseline: $b^{(k)} \leftarrow \frac{1}{N} \sum_{i=1}^N R(\pi^{(k,i)})$
            \ENDFOR
            \STATE Collect all trajectories $\{\pi^{(k,i)}\}_{k=1}^{K}\,{}_{i=1}^{N}$ and view-specific baselines $\{b^{(k)}\}_{k=1}^{K}$
            \STATE Compute joint gradient $\nabla_\theta \mathcal{J}_{\mathcal{G}}(\theta)$ via Eq.~\eqref{eq:collective_gradient}
            \STATE Update parameters: $\theta \leftarrow \theta + \eta \nabla_\theta \mathcal{J}_{\mathcal{G}}(\theta)$
        \ENDFOR
        \STATE {\bfseries Output:} Optimized invariant policy $\theta^*$
    \end{algorithmic}
\end{algorithm}

\section{Implementation Details}
\label{app:impl}

We adopt an encoder-decoder architecture \cite{kool2018attention} with $L=6$ layers, embedding dimension $d=256$, feed-forward dimension $d_{\text{ff}}=512$, and $M=8$ attention heads \cite{vaswani2017attention}. The decoder follows the context-based attention mechanism of POMO \cite{kwon2020pomo}. All training is performed on a single NVIDIA RTX 4090 GPU using the Adam optimizer \cite{kingma2014adam} with learning rate $\eta=10^{-4}$ and weight decay $10^{-6}$. The model is trained for 2,000 epochs, each covering 100,000 randomly generated instances. MViewRouter takes 5.3 min per epoch on TSP (vs.\ 5.0 min for POMO and 9.0 min for POMO+PO \cite{pan2025preference}).

\section{CVRP-100 Training Dynamics}
\label{app:training_curves}

Figure~\ref{fig:training_curves_cvrp} shows the training curves on CVRP-100 under the same two evaluation protocols as Section~\ref{sec:training_dynamics}: original view and $\times$8 augmentation. MViewRouter converges faster and to a lower tour length than POMO \cite{kwon2020pomo} under both protocols, mirroring the TSP-100 behavior. In the original-view setting, MViewRouter reaches the performance level of augmented POMO well before training converges, demonstrating that the geometric invariance learned by MAA and CPGA transfers directly to a structurally different problem without any architectural modification. The consistent advantage across both TSP and CVRP confirms that the convergence benefit stems from the MAA and CPGA mechanisms rather than TSP-specific properties.

\begin{figure*}[h]
    \centering
    \includegraphics[width=0.95\textwidth]{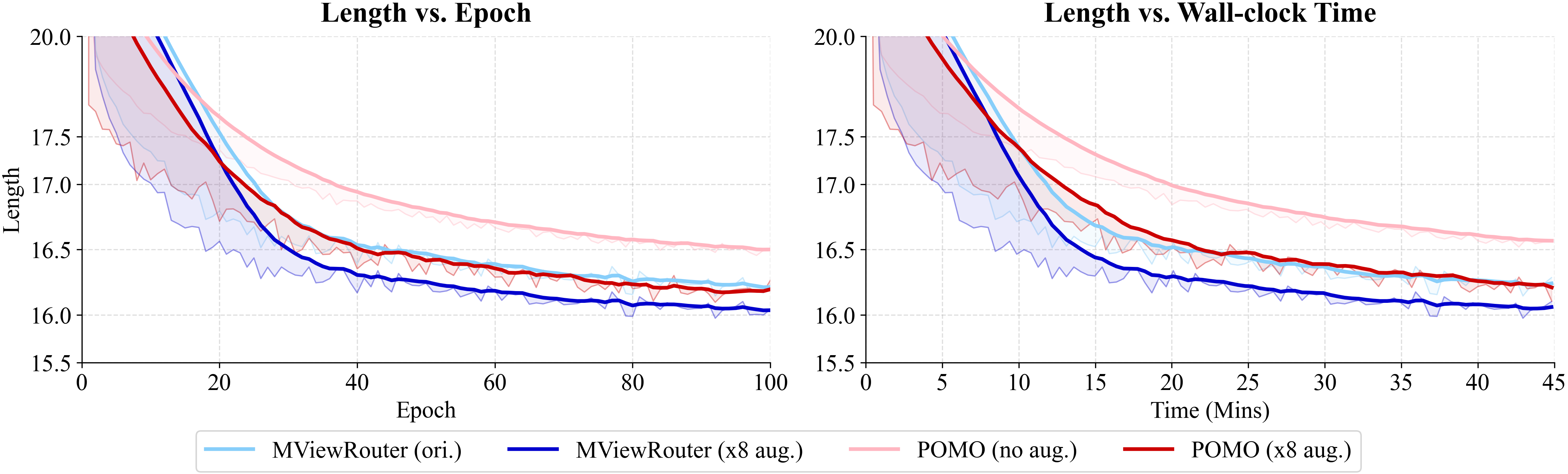}
    \caption{Training curves on CVRP-100: POMO and MViewRouter under original view and $\times$8 aug.}
    \label{fig:training_curves_cvrp}
\end{figure*}

\section{Full TSPLIB Benchmark Results}
\label{app:tsplib}

Table~\ref{tab:TSPLIB_full} reports per-instance results on all 29 TSPLIB instances (50--200 nodes) \cite{reinelt1991tsplib}. Both POMO \cite{kwon2020pomo} and MViewRouter use $\times$8 augmentation with a model trained on TSP-100, evaluated zero-shot without retraining or fine-tuning.

\paragraph{Overall and per-instance analysis.}
MViewRouter improves over POMO on 22 of 29 instances, with a 25.5\% relative reduction in mean optimality gap (2.373\% $\to$ 1.769\%). The gains are most pronounced on instances where POMO's gap is large, suggesting that MViewRouter's internalized geometric equivariance is especially beneficial when the coordinate orientation of a real-world instance happens to be unfavorable for the base policy. Notable examples include \texttt{bier127} (6.66\% $\to$ 0.77\%), \texttt{rat99} (3.93\% $\to$ 1.81\%), and \texttt{krA100} (0.96\% $\to$ 0.016\%). The \texttt{kr*} family (Krolak instances with clustered structure) shows consistent improvement across all sizes from 100 to 200 nodes.

\paragraph{Scaling behavior.}
Performance on instances near $n{=}200$ reflects the well-known distribution shift challenge: the model is trained on $n{=}100$, and both methods see their gaps widen at double the scale. Nevertheless, MViewRouter maintains an advantage on the majority of larger instances (\texttt{d198}: 23.49\% $\to$ 18.23\%; \texttt{krB200}: 3.93\% $\to$ 3.49\%), consistent with the zero-shot generalization trend in Figure~\ref{fig:generalization} showing that MViewRouter's geometric inductive bias continues to help even under significant scale shift.

\begin{table}[h]
\centering
\renewcommand{\arraystretch}{1.0}
\caption{Per-instance results on real-world TSPLIB instances (50--200 nodes). Both POMO and MViewRouter utilize $\times$8 augmentations.}
\label{tab:TSPLIB_full}
\fontsize{8pt}{11pt}\selectfont
\setlength{\tabcolsep}{1.7pt}
\begin{tabular}{l c c ccc ccc | l c c ccc ccc}
\toprule
\multirow{2}{*}{Inst.} & \multirow{2}{*}{Size} & \multirow{2}{*}{Opt.} & \multicolumn{3}{c}{POMO} & \multicolumn{3}{c}{\textbf{MViewRouter}} & \multirow{2}{*}{Inst.} & \multirow{2}{*}{Size} & \multirow{2}{*}{Opt.} & \multicolumn{3}{c}{POMO} & \multicolumn{3}{c}{\textbf{MViewRouter}} \\
\cmidrule(lr){4-6}\cmidrule(lr){7-9} \cmidrule(lr){13-15}\cmidrule(lr){16-18}
 & & & Len. & Gap & T. (s) & Len. & Gap & T. (s) & & & & Len. & Gap & T. (s) & Len. & Gap & T. (s) \\
\midrule
eil51 & 51 & 6.762 & 6.818 & 0.827 & 0.224 & \textbf{6.809} & \textbf{0.698} & 0.265 & pr124 & 124 & 6.263 & 6.269 & 0.101 & 0.038 & \textbf{6.263} & \textbf{0.003} & 0.043 \\
berlin52 & 52 & 4.398 & 4.399 & 0.028 & 0.017 & \textbf{4.399} & \textbf{0.024} & 0.026 & bier127 & 127 & 6.937 & 7.399 & 6.660 & 0.034 & \textbf{6.991} & \textbf{0.774} & 0.046 \\
st70 & 70 & 6.818 & 6.839 & 0.315 & 0.020 & 6.840 & 0.328 & 0.028 & ch130 & 130 & 8.831 & 8.868 & 0.417 & 0.040 & \textbf{8.835} & \textbf{0.040} & 0.050 \\
eil76 & 76 & 7.472 & 7.561 & 1.187 & 0.022 & \textbf{7.561} & \textbf{1.187} & 0.028 & pr136 & 136 & 8.545 & 8.631 & 1.006 & 0.045 & \textbf{8.602} & \textbf{0.666} & 0.047 \\
pr76 & 76 & 5.518 & 5.520 & 0.028 & 0.022 & 5.521 & 0.046 & 0.028 & pr144 & 144 & 5.358 & 5.407 & 0.919 & 0.056 & \textbf{5.370} & \textbf{0.232} & 0.049 \\
rat99 & 99 & 5.685 & 5.909 & 3.931 & 0.027 & \textbf{5.788} & \textbf{1.809} & 0.036 & ch150 & 150 & 9.337 & 9.398 & 0.651 & 0.058 & \textbf{9.383} & \textbf{0.492} & 0.052 \\
krA100 & 100 & 5.407 & 5.459 & 0.960 & 0.027 & \textbf{5.408} & \textbf{0.016} & 0.035 & krA150 & 150 & 6.710 & 6.803 & 1.380 & 0.058 & \textbf{6.771} & \textbf{0.908} & 0.054 \\
krB100 & 100 & 5.627 & 5.674 & 0.840 & 0.031 & \textbf{5.643} & \textbf{0.287} & 0.035 & krB150 & 150 & 6.657 & 6.765 & 1.616 & 0.044 & \textbf{6.764} & \textbf{1.607} & 0.053 \\
krC100 & 100 & 5.286 & 5.292 & 0.115 & 0.028 & \textbf{5.287} & \textbf{0.015} & 0.035 & pr152 & 152 & 5.368 & 5.450 & 1.536 & 0.039 & \textbf{5.419} & \textbf{0.953} & 0.052 \\
krD100 & 100 & 5.461 & 5.504 & 0.793 & 0.028 & \textbf{5.479} & \textbf{0.323} & 0.035 & u159 & 159 & 8.092 & 8.186 & 1.165 & 0.041 & \textbf{8.173} & \textbf{1.002} & 0.055 \\
krE100 & 100 & 5.624 & 5.704 & 1.426 & 0.028 & \textbf{5.662} & \textbf{0.679} & 0.035 & rat195 & 195 & 8.038 & 8.746 & 8.814 & 0.050 & 8.815 & 9.663 & 0.066 \\
rd100 & 100 & 8.065 & 8.073 & 0.103 & 0.028 & \textbf{8.065} & \textbf{0.003} & 0.035 & d198 & 198 & 3.917 & 4.837 & 23.487 & 0.049 & \textbf{4.631} & \textbf{18.226} & 0.069 \\
eil101 & 101 & 8.500 & 8.666 & 1.949 & 0.033 & \textbf{8.660} & \textbf{1.877} & 0.038 & krA200 & 200 & 7.452 & 7.667 & 2.879 & 0.059 & 7.742 & 3.885 & 0.068 \\
lin105 & 105 & 4.755 & 4.789 & 0.709 & 0.029 & 4.802 & 0.996 & 0.037 & krB200 & 200 & 7.466 & 7.760 & 3.932 & 0.049 & \textbf{7.726} & \textbf{3.487} & 0.068 \\
pr107  & 107 & 5.370 & 5.427 & 1.055 & 0.030 & 5.428 & 1.083 & 0.038 & \cellcolor[gray]{0.9} Mean & \cellcolor[gray]{0.9}- & \cellcolor[gray]{0.9}- & \cellcolor[gray]{0.9}6.683 & \cellcolor[gray]{0.9}2.373 & \cellcolor[gray]{0.9}0.043 & \cellcolor[gray]{0.9}\textbf{6.650} & \cellcolor[gray]{0.9}\textbf{1.769} & \cellcolor[gray]{0.9}0.052 \\
\bottomrule
\end{tabular}
\end{table}


\end{document}